\documentclass[boldtwoside,11pt]{article}

\usepackage{jmlr2e}
\usepackage{hyperref}
\usepackage{amsmath,,amsfonts}
\usepackage{tabularx}
\usepackage{dsfont}
\usepackage{stmaryrd}
\usepackage{cancel}
\usepackage{color}
\usepackage{makecell}
\usepackage{multicol}

\usepackage{blkarray}
\usepackage{rotating} 
\newcommand{\colmatindex}[1]{\begin{turn}{-80}\mbox{\color{blue}\small#1}\end{turn}}
\newcommand{\rowmatindex}[1]{\begin{turn}{-5}\mbox{\color{blue}\scriptsize#1}\end{turn}}
\usepackage{wasysym,MnSymbol}
\newcommand{\demi}{\parbox{0.2cm}{\scriptsize 1/2}}

\newcommand{\loss}{\square}
\newcommand{\tie}{\diamond}
\newcommand{\win}{\blacksquare}

\newtheorem{mydef}{Definition}
\usepackage{comment}
\specialcomment{TUcom}{\color{red} $\langle\texttt{\bf TU}\rangle$ }{ $\langle\texttt{\bf \slash{}TU}\rangle$\color{black}}
\specialcomment{PGcom}{\color{blue} $\langle\texttt{\bf PG}\rangle$ }{ $\langle\texttt{\bf \slash{}PG}\rangle$\color{black}}

\usepackage{mathrsfs}

\ShortHeadings{Utlity-based Dueling Bandits as PM game}{Gajane}
\firstpageno{1}

\begin{document}

\title{Utlity-based Dueling Bandits as a  Partial Monitoring Game}
\author{\name Pratik Gajane \email pratik.gajane@inria.fr\\
       \addr INRIA SequeL team, Lille 59650, France\\
       Orange labs, Lannion 22300, France\\
       \AND
       \name Tanguy Urvoy  \email tanguy.urvoy@orange.com \\
       \addr Orange labs,
       Lannion 22300,
       France}
\editor{}
\maketitle
\begin{abstract}
Partial monitoring is a generic framework for sequential decision-making with incomplete feedback. It encompasses a wide class of problems such as dueling bandits,  learning with expect advice, dynamic pricing, dark pools, and label efficient prediction. We study the utility-based dueling bandit problem as an instance of partial monitoring problem and prove that it fits the time-regret partial monitoring hierarchy as an {\it easy} -- i.e. $\tilde{\Theta}\left(\sqrt{T}\right)$ -- instance.
We survey some partial monitoring algorithms and see how they could be used to solve dueling bandits efficiently.
\end{abstract}


\section{Introduction}
Partial Monitoring (PM) provides a generic mathematical model for sequential decision-making with incomplete feedback. It is a recent paradigm in the reinforcement learning community. Similarly the multi-armed bandit problem is a classical mathematical model for the \textit{exploration/exploitation} dilemma inherent in reinforcement learning \citep[see][]{DBLP:journals/ftml/BubeckC12}. The \textit{K-armed dueling bandit problem} \citep{Yue:2009:IOI:1553374.1553527} is a variation of the multi-armed bandit problem where two arms are selected at each round with a relative feedback.


 Several generic partial monitoring algorithms have been proposed for both stochastic and adversarial settings \citep[see][for details]{DBLP:journals/mor/BartokFPRS14}.
 With the exception of  {\sc globalexp3} \cite{Bartok2013} which tries to capture the structure of the games  more finely, these algorithms only focus on the time bound and perform inefficiently in term of the number of actions.
 As we show in section \ref{PM_algorithms}, for a dueling bandit problem, the number of actions is quadratic in the number of arms $K$ and these algorithms, including {\sc globalexp3}, provide at best a $\tilde{\mathcal{O}}\left(K\sqrt{T}\right)$ regret guarantee whereas a dedicated algorithm like {\sc rex3} \citep{GajaneICML2015} can provide a $\tilde{\mathcal{O}}\left(\sqrt{KT}\right)$ guarantee\footnote{The $\tilde{\mathcal{O}}\left(\cdot\right)$ notation hides logarithmic factors.}.
Studying partial monitoring algorithms from the perspective of dueling bandits is hence an interesting and challenging problem which could help us improve the ability of PM algorithms to capture the structure of sequential decision problems in a better way.

In this preliminary work, we investigate how a utility-based dueling bandits problem can be modeled as an instance of a partial monitoring game. Our main contribution is that, we prove, using the PM formalism, that it is an 
{\it easy PM  instance} according to the hierarchy defined in \cite{DBLP:journals/mor/BartokFPRS14}.
Furthermore, we take a brief look at the existing partial monitoring algorithms and examine how they could be used to solve dueling bandits problems efficiently.

\subsection{Dueling bandits}
The K-armed dueling bandit problem is a variation of the classical multi-armed bandit
problem introduced by \cite{Yue:2009:IOI:1553374.1553527} 
to formalize the exploration/exploitation dilemma in learning from preference feedback.
In its utility-based formulation, at each time period, 
the environment sets a bounded value for each of the $K$ arms and simultaneously the learner selects two arms. The learner only sees the outcome of the \textit{duel} between the selected arms (i.e. the feedback indicates which of the selected arms has better value) and receives
the average of the gains of the selected arms. The goal of the learner is to maximize her \textit{cumulative gain}.

Relative feedback is naturally suited to many practical applications because users are more obliging to provide a relative preference feedback rather than an absolute feedback e.g. compared to ``I rate Tennis at $32/50$ and Football at $48/50$" (absolute feedback) , it's easier for users to say ``I like Football more than Tennis" (relative feedback).
Information Retrieval systems with \textit{implicit feedback} are another important application of the dueling bandits \citep[see][]{RadlinskiJ07}.
The major difficulty of the dueling bandit problem is that the learner cannot directly observe the loss (or gain) of the selected actions. To capture this aspect of the problem, it can be modeled as an instance of the \textit{partial monitoring problem} as defined by \cite{conf/colt/PiccolboniS01}.

\subsection{Partial monitoring games}
A partial monitoring game is defined by a tuple $\left(\boldsymbol{N},\boldsymbol{M},\boldsymbol{\Sigma}, \mathcal{L}, \mathcal{H}\right)$ \footnote{Uppercase boldface letters are used to denote sets} where 
$\boldsymbol{N}$, $\boldsymbol{M}$, and $\mathbf{\Sigma}$ are the {\it action} set, the {\it outcome} set, and the {\it feedback alphabet} respectively.
To each action $I\in \boldsymbol{N}$ and outcome $J\in\boldsymbol{M}$, the {\it loss function} 
$\mathcal{L}$ associates a real-valued loss $\mathcal{L}(I,J)$ and the 
{\it feedback function} $\mathcal{H}$ associates
a feedback symbol $\mathcal{H}(I,J) \in \boldsymbol{\Sigma}$.

\paragraph{}
In every round, the opponent and the learner simultaneously choose an outcome $J_t$ from $\boldsymbol{M}$ and an action $I_t$ from $\boldsymbol{N}$, respectively. The learner then suffers the loss $\mathcal{L}(I_t, J_t)$ and receives the feedback $ \mathcal{H}(I_t, J_t)$. Only the feedback is revealed to the learner, the outcome and the loss remain hidden. In some problems, gain $\mathcal{G}$ is considered instead of loss. The loss function $\mathcal{L}$ and the feedback function $\mathcal{H}$ are known to the learner. When both $\boldsymbol{N}$ and $\boldsymbol{M}$ are finite, the loss function and the feedback function can be encoded by matrices, namely loss matrix and feedback matrix each of size $|\boldsymbol{N}| \times |\boldsymbol{M}|$.
The aim of the learner is to control the expected cumulative regret against the best single-action (or pure) strategy at time $T$:

$$ R_T=\max\limits_i\sum_{t=1}^T \mathcal{L}(I_t,J_t) - \mathcal{L}(i,J_t) $$

Various interesting problems can be modeled as partial monitoring games, such as learning with expect advice (\cite{Littlestone1994}), the
multi-armed bandit problem (\cite{auer2002nonstochastic}), dynamic pricing (\cite{conf/focs/KleinbergL03}), the dark pool problem (\cite{AgarwalBD10}), label efficient prediction (\cite{Cesa-bianchi05}), and linear and convex optimization with full or bandit feedback (\cite{Zinkevich03}, \cite{AbernethyHR08}, \cite{cs-LG-0408007}). We shall briefly explain a couple of examples: 
%
\paragraph{The dynamic pricing problem:}
A seller has a product to sell and the customers wish to buy it. 
At each time period, the customer secretly decides on a maximum amount she is willing to pay and the seller sets a selling price.
If the selling price is below the maximum amount the buyer is willing to pay, she buys the product and the seller's gain is the selling price she fixed. If the selling price is too expensive, her gain is zero.
The feedback is partial because the seller only recieves a binary information stating whether the customer has bought the product or not.
A PM formulation of this problem is provided below:
$$x\in \boldsymbol{N} \subseteq \mathbb{R}, \quad y\in \boldsymbol{M} \subseteq \mathbb{R} , \quad \mathbf{\Sigma} = \{\text{``sold"}, \text{``not sold"}\}$$
\begin{center}
\begin{tabular}{r@{\quad}l}
$\begin{aligned}
\mathcal{G}(x,y) = \begin{cases}
 		 0, & \text{if } x  > y, \\
		  x, & \text{if } x \leq y,
		\end{cases}
\end{aligned}$ 
 \quad \quad
$\begin{aligned}
\mathcal{H}(x,y) = \begin{cases}
 		 \text{``not sold"}, & \text{if } x  > y, \\
 		 \text{``sold"}, & \text{if } x \leq y,
		\end{cases}
\end{aligned}$
\end{tabular}
\end{center}

\paragraph{The multi-armed bandit problem:}
At each time period, the learner pulls one of the $K$ arms and receives it's
corresponding gain which is bounded in $[0,1]$. The learner sees only her gain and not the
gain of other arms. The learner's goal is to win almost as much as the optimal arm.
A partial monitoring  formulation of this problem is provided with
a set of K arms/actions $i\in \boldsymbol{N} = \{1, \dots, K\}$, an alphabet $\mathbf{\Sigma} = [0,1]$, and a set of 
environment outcomes which are vectors\footnote{Lowercase boldface letters are used to denote vectors} $\boldsymbol{m} \in \boldsymbol{M}= [0,1]^K$. The entry with index $i$ ($\boldsymbol{m}_i$) denotes the instantaneous gain of the $i^{th}$ arm.
Assuming binary gains, $\boldsymbol{M}$ is finite and of size $2^K$.
\begin{center}
\begin{tabular}{r@{\quad}l}
$\begin{aligned}
\mathcal{G}(i,\boldsymbol{m}) = \boldsymbol{m}_i
\end{aligned}
$\quad \quad$
\begin{aligned}
\mathcal{H}(i,\boldsymbol{m}) = \boldsymbol{m}_i
\end{aligned}$
\end{tabular}
\end{center}


\section{Dueling bandits as a Partial Monitoring game}
\label{sec:DBasPM}

The utility-based dueling bandits model is similar to multi-armed bandits but the action sets differ.
An action consists here of selecting a pair $(i,j)$ of arms. However, symmetric actions like $(i,j)$ and $(j,i)$ lead to the same gains and provide equally informative feedback.
Hence the action set for the learner can be restricted to
$\boldsymbol{N} = \left\{(i,j): 1 \leq i,j \leq K, i \leq j \right\}$.
When the environment selects an outcome $\boldsymbol{m} \in \boldsymbol{M}$  and the learner selects a duel/action $(i,j) \in \boldsymbol{N}$, the instantaneous gain $\mathcal{G}((i,j),\boldsymbol{m}) $ and feedback $\mathcal{H}((i,j),\boldsymbol{m})$ are as follows: \\

\begin{tabular}{cc}
$\begin{aligned}
 \mathcal{G}((i,j),\boldsymbol{m}) = \frac{\boldsymbol{m}_i+ \boldsymbol{m}_j}{2}
\end{aligned}$ 
 \quad \quad
$\begin{aligned}
\mathcal{H}((i,j),\boldsymbol{m}) = \begin{cases}
  \loss & \text{if } \boldsymbol{m}_i < \boldsymbol{m}_j \quad\text{\it (loss)}\\
  \tie & \text{if } \boldsymbol{m}_i = \boldsymbol{m}_j \quad\text{\it (tie)}\\
  \win  & \text{if } \boldsymbol{m}_i > \boldsymbol{m}_j \quad\text{\it (win)}
	\end{cases} \\
\end{aligned}$ 
\end{tabular}

\paragraph{}
To illustrate this formalism, we encode a $4$-armed binary-gain dueling bandit problem as a PM problem in Figure~\ref{fig:matrices}.  The first element of every column is of the form ${\color{blue} \boldsymbol{m}_1 \boldsymbol{m}_2 \boldsymbol{m}_3 \boldsymbol{m}_4}$ where $\boldsymbol{m}_i$ is the gain for $i^{th}$ arm. The first element of every row is of the form ${\color{blue}d_1d_2}$ where $d_1$ is the first arm being picked and $d_2$ being the second.

\begin{figure}
\[
\mathcal{G}=
\arraycolsep=4pt
\def\arraystretch{1.05}
\begin{array}{cccccccccccccccccc}
&\colmatindex{0000} &\colmatindex{0001} &\colmatindex{0010} 
&\colmatindex{0011} &\colmatindex{0100} &\colmatindex{0101} &
\colmatindex{0110} &\colmatindex{0111} &\colmatindex{1000} &
\colmatindex{1001} &\colmatindex{1010} &\colmatindex{1011} &
\colmatindex{1100} &\colmatindex{1101} &\colmatindex{1110} &
\colmatindex{1111}\\[+0.5cm]
\rowmatindex{11}& 0 & 0 & 0 & 0 & 0 & 0 & 0 & 0 & 1 & 1 & 1 & 1 &1 &1 &1 &1 \\
\rowmatindex{12}& 0 & 0 & 0 & 0 &\demi &\demi &\demi &\demi &\demi &\demi &\demi &\demi &1 &1 &1 &1 \\
\rowmatindex{13}& 0 & 0 &\demi &\demi & 0 & 0 &\demi &\demi &\demi &\demi &1 &1 &\demi &\demi &1 &1 \\
\rowmatindex{14}& 0 &\demi & 0 &\demi & 0 &\demi & 0 &\demi &\demi &1 &\demi &1 &\demi &1 &\demi &1 \\
\rowmatindex{22}& 0 & 0 & 0 & 0 & 1 & 1 & 1 & 1 & 0 & 0 & 0 & 0 &1 &1 &1 &1 \\
\rowmatindex{23}& 0 & 0 &\demi &\demi &\demi &\demi &1 &1 & 0 & 0 &\demi &\demi &\demi &\demi &1 &1 \\
\rowmatindex{24}& 0 &\demi & 0 &\demi &\demi &1 &\demi &1 & 0 &\demi & 0 &\demi &\demi &1 &\demi &1 \\
\rowmatindex{33}& 0 & 0 & 1 & 1 & 0 & 0 & 1 & 1 & 0 & 0 & 1 & 1 &0 &0 &1 &1 \\
\rowmatindex{34}& 0 &\demi &\demi &1 & 0 &\demi &\demi &1 & 0 &\demi &\demi &1 & 0 &\demi &\demi &1 \\ 
\rowmatindex{44}& 0 & 1 & 0 & 1 & 0 & 1 & 0 & 1 & 0 & 1 & 0 & 1 &0 &1 &0 &1 \\
\end{array}
\]
\[
\mathcal{H}=
\arraycolsep=4pt
\def\arraystretch{1.05}
\begin{array}[c]{cccccccccccccccccc}
&\colmatindex{0000} &\colmatindex{0001} &\colmatindex{0010} 
&\colmatindex{0011} &\colmatindex{0100} &\colmatindex{0101} &
\colmatindex{0110} &\colmatindex{0111} &\colmatindex{1000} &
\colmatindex{1001} &\colmatindex{1010} &\colmatindex{1011} &
\colmatindex{1100} &\colmatindex{1101} &\colmatindex{1110} &
\colmatindex{1111}\\[+0.5cm]
\rowmatindex{11}&\tie &\tie &\tie &\tie &\tie &\tie &\tie &\tie & \tie & \tie & \tie & \tie &\tie &\tie &\tie &\tie \\
\rowmatindex{12}&\tie & \tie & \tie & \tie & \loss & \loss & \loss & \loss & \win & \win & \win & \win & \tie & \tie & \tie & \tie \\
\rowmatindex{13}&\tie & \tie & \loss & \loss & \tie & \tie & \loss & \loss & \win & \win & \tie & \tie & \win & \win & \tie & \tie \\
\rowmatindex{14}&\tie & \loss & \tie & \loss & \tie & \loss & \tie & \loss & \win & \tie & \win & \tie & \win & \tie & \win & \tie \\
\rowmatindex{22}&\tie &\tie &\tie &\tie &\tie &\tie &\tie &\tie & \tie & \tie & \tie & \tie &\tie &\tie &\tie &\tie \\
\rowmatindex{23}&\tie & \tie & \loss & \loss & \win & \win & \tie & \tie & \tie & \tie & \loss & \loss & \win & \win & \tie & \tie \\
\rowmatindex{24}&\tie & \loss & \tie & \loss & \win & \tie & \win & \tie & \tie & \loss & \tie & \loss & \win & \tie & \win & \tie \\
\rowmatindex{33}&\tie &\tie &\tie &\tie &\tie &\tie &\tie &\tie & \tie & \tie & \tie & \tie &\tie &\tie &\tie &\tie \\
\rowmatindex{34}&\tie & \loss & \win & \tie & \tie & \loss & \win & \tie & \tie & \loss & \win & \tie & \tie & \loss & \win & \tie \\
\rowmatindex{44}&\tie &\tie &\tie &\tie &\tie &\tie &\tie &\tie & \tie & \tie & \tie & \tie &\tie &\tie &\tie &\tie \\
\end{array}
\]

\caption{Gain matrix $\mathcal{G}$ and feedback matrix $\mathcal{H}$ for a $4$-armed binary dueling bandits resulting in $10$ non-duplicate actions and $16$ possible outcomes.}
\label{fig:matrices}
\end{figure}

\begin{figure}
\[
\mathcal{S}_{(12)}=
\arraycolsep=4pt
\def\arraystretch{1.05}
\begin{array}[c]{cccccccccccccccccc}
&\colmatindex{0000} &\colmatindex{0001} &\colmatindex{0010} 
&\colmatindex{0011} &\colmatindex{0100} &\colmatindex{0101} &
\colmatindex{0110} &\colmatindex{0111} &\colmatindex{1000} &
\colmatindex{1001} &\colmatindex{1010} &\colmatindex{1011} &
\colmatindex{1100} &\colmatindex{1101} &\colmatindex{1110} &
\colmatindex{1111}\\[+0.5cm]
\rowmatindex{$\loss$}&0 & 0 & 0 & 0 & 1 & 1 & 1 & 1 & 0 & 0 & 0 & 0 & 0 & 0 & 0 & 0 \\
\rowmatindex{$\tie$}&1 & 1 & 1 & 1 & 0 & 0 & 0 & 0 & 0 & 0 & 0 & 0 & 1 & 1 & 1 & 1 \\
\rowmatindex{$\win$}&0 & 0 & 0 & 0 & 0 & 0 & 0 & 0 & 1 & 1 & 1 & 1 & 0 & 0 & 0 & 0 \\
\end{array}
\]
\caption{Signal matrix for action $(12)$ for the same problem as in Figure~\ref{fig:matrices}.}
\label{fig:smatrix}
\end{figure}

\section{Hierarchy and basic concepts of partial monitoring problems}
\label{sec:hierarchy}
In this section, firstly, we take a brief review of the basic concepts of partial monitoring problems. Most of the definitions in this section are taken from \cite{Bartók11minimaxregret} and \cite{Bartok2013}.

Consider a finite partial monitoring game with action set $\boldsymbol{N}$, outcome set $\boldsymbol{M}$, loss matrix $\mathcal{L}$ and feedback matrix $\mathcal{H}$.
For any action $i \in \boldsymbol{N}$, loss vector $\boldsymbol{l}_i$ denotes the column vector consisting of $i^{th}$ row in $\mathcal{L}$. Correspondingly, gain vector $\boldsymbol{g}_i$ denotes the column vector consisting of $i^{th}$ row in $\mathcal{G}$. For the rest of the article, gain vector $\boldsymbol{g}_i$ and loss vector $\boldsymbol{l}_i$ will be used interchangeably depending upon the setting. Let $\Delta_{|\boldsymbol{M}|}$ be the $|\boldsymbol{M}|\!-\!1$-dimensional probability simplex i.e. 
$\Delta_{|\boldsymbol{M}|} =  \left\{ \boldsymbol{q} \in [0,1]^{|\boldsymbol{M}|}\ |\  || \boldsymbol{q}||_1 = 1 \right\}$. For any outcome sequence of length $T$, the vector $\boldsymbol{q}$ denoting the relative frequencies with which each outcome occurs is in $\Delta_{|\boldsymbol{M}|}$.  The cumulative loss of action $i$ for this outcome sequence can hence be described as follows:
$$ \sum_{t=1}^{T} \mathcal{L}(i, J_t) =  T \cdot  \boldsymbol{l}_i^\intercal \boldsymbol{q}$$ \\
The vectors denoting the outcome frequencies can be thought of as the opponent strategies. These opponent strategies determine which action is optimal i.e. the action with the lowest cumulative loss. This induces a \textit{cell decomposition} on $\Delta_{|\boldsymbol{M}|}$. 
\begin{mydef}[Cells]
The cell of an action $i$ is defined as 
$$ C_i = \left\{  \boldsymbol{q} \in \Delta_{|\boldsymbol{M}|}\ |\ \boldsymbol{l}_i^\intercal \boldsymbol{q} = \min_{j \in \boldsymbol{M}}  \boldsymbol{l}_j^\intercal \boldsymbol{q} \right\}
$$ 
\end{mydef}
In other words, a cell of an action consists of those opponent strategies in the probability simplex for which it is the optimal action.
An action $i$ is said to be \textit{Pareto-optimal} if there exists an opponent strategy $\boldsymbol{q}$ such
that the action $i$ is optimal under $\boldsymbol{q}$. 
The actions whose cells have a positive $(|\boldsymbol{M}| - 1)$-dimensional volume are called \textit{Strongly Pareto-optimal}. 
Actions that are Pareto-optimal but not strongly Pareto-optimal are called \textit{degenerate}.
\begin{mydef}[Cell decomposition]
The cells of strongly Pareto-optimal actions form a finite cover of $\Delta_M$ called as the {\it cell-decomposition}.
\end{mydef}

Two actions cells $i$ and $j$ from the cell decomposition are {\it neighbors} if their intersection is an $(|\boldsymbol{M}| - 2)$-dimensional polytope. The actions corresponding to these cells are also called as {\it neighbors}. 
The raw feedback matrices can be `standardized' by encoding their symbols in {\it signal matrices}:
%
%
%
\begin{mydef}[Signal matrices]
\label{def:signal}
For an action $i$, let $\sigma_1, \dots, \sigma_{s_i} \in \Sigma$ be the symbols occurring in row $i$ of $\mathcal{H}$.
The signal matrix $\mathcal{S}_i$ of action $i$ is defined as the incidence matrix of symbols and outcomes i.e.
$\mathcal{S}_i(k,m) = \llbracket{}\mathcal{H}(i, m) = \sigma_k\rrbracket{} \quad k = 1, \dots, {s_i}, \quad\text{for }  m \in \boldsymbol{M}$ \footnote{we use $\llbracket{}\cdot\rrbracket{}$ to denote the indicator function}.
\end{mydef}
{\it Observability} is a key notion to assess the difficulty of a PM problem in terms of regret $R_T$ against best action at time $T$.

\begin{mydef}[Observability] For actions $i$ and $j$, we say that 
$\boldsymbol{l}_i - \boldsymbol{l}_j $
is \emph{globally observable} if 
$\boldsymbol{l}_i - \boldsymbol{l}_j  \in \operatorname{Im} \mathcal{S}^\intercal$.
Where the global signal matrix $\mathcal{S}$ is obtained by stacking all signal matrices.
Furthermore, if $i$ and $j$ are neighboring actions, then $\boldsymbol{l}_i - \boldsymbol{l}_j $ is called \emph{locally observable} if
$\boldsymbol{l}_i - \boldsymbol{l}_j  \in \operatorname{Im} \mathcal{S}_{i,j}^\intercal$ where the local signal matrix $\mathcal{S}_{i,j}$ is obtained by stacking the signal matrices of all neighboring actions for $i,j$: $\mathcal{S}_{k}$ for $k\in \left\{ k \in \boldsymbol{N}\ |\ C_i \cap C_j  \subseteq C_k \right\}$.
\label{obs}
\end{mydef}

\begin{theorem}[Classification of partial monitoring problems] 
Let $\left(\boldsymbol{N},\boldsymbol{M},\boldsymbol{\Sigma}, \mathcal{L},\mathcal{H}
\right)$ be a partial monitoring game.
Let $ \{ C_1, \dots, C_k\}$ be it's cell decomposition, with corresponding loss vectors $\boldsymbol{l}_1, \dots, \boldsymbol{l}_k$.
The game falls into the following four regret categories. 
\begin{itemize}
\item $R_T = 0$ if there exists an action with $C_i = \Delta_{|\boldsymbol{M}|}$. This case is called \emph{trivial}. 
\item $R_T \in \Theta(T)$ if there exist two strongly Pareto-optimal actions $i$ and $j$ such that $\boldsymbol{l}_i - \boldsymbol{l}_j$ is not globally observable. This case is called \emph{hopeless}.
\item $R_T \in \tilde\Theta(\sqrt T)$ if it is not trivial and for all pairs of (strongly Pareto-optimal) neighboring actions $i$ and $j$, $\boldsymbol{l}_i - \boldsymbol{l}_j$ is locally observable. This case is called \emph {easy}.
\item $R_T \in \Theta(T^{2/3})$ if $\mathscr{G}$ is not hopeless and there exists a pair of neighboring actions $i$ and $j$ such that  $\boldsymbol{l}_i - \boldsymbol{l}_j$ is not locally observable. This case is called \emph {hard}. 
\end{itemize}
\label{thm1}
\end{theorem}

\section{Dueling bandits in the partial monitoring hierarchy}
This section examines the place of the dueling bandit problem in the hierarchy of partial monitoring problems described above.
Note that the existence of the {\sc rex3} algorithm  \citep{GajaneICML2015} with a $\tilde{\Theta}\left(\sqrt{KT}\right)$ regret guarantee is enough to state that dueling bandit is an {\it easy game} according to the hierarchy described in 
 Theorem~\ref{thm1}, but our aim here is to retrieve this result from the PM machinery.

\begin{theorem}[Duelings bandits: locally observable] In a binary utility-based dueling bandit problem with more than two arms, all the pairs of actions are locally observable. 
\label{mainthm}
\end{theorem}
\begin{proof}
Consider a dueling bandit problem as defined in Section~\ref{sec:DBasPM} with binary gains and $K>2$ arms. The signal matrix of any action $(i,j) \in \boldsymbol{N}$ is defined as follows:
\begin{center}
\begin{tabular}{ccc}
$S_{(i,j)}(\loss, \boldsymbol{m}) = \llbracket{}\boldsymbol{m}_i < \boldsymbol{m}_j \rrbracket{}
$,\ 
&
$S_{(i,j)}(\tie, \boldsymbol{m}) =
\llbracket{}\boldsymbol{m}_i = \boldsymbol{m}_j \rrbracket{}
$,\ 
&
$S_{(i,j)}(\win, \boldsymbol{m}) =
\llbracket{}\boldsymbol{m}_i > \boldsymbol{m}_j \rrbracket{}
$
\end{tabular}
\end{center}

\noindent
In the following, we show that for any pair of actions $(i,j)$ and $(i',j')$, $\textbf{\textit{g}}_{(i',j')} - \textbf{\textit{g}}_{(i',j')}$ is locally observable. For the sake of readability, let's consider $S^\win$, $S^\tie$ and $S^\loss$ to be the column vectors containing the rows pertaining to the symbols $\win$, $\tie$ and $\loss$ of the signal matrix $S$ respectively. We consider the following two cases for the pair of actions which together cover all the possibilities:
\begin{itemize}
\item A pair of actions that share at-least one common arm: 
	\begin{enumerate}
		\item Actions $(i,k)$ and $(k,j)$. For any binary gain outcome $\boldsymbol{m}$, we have :
		\begin{align} 
			\boldsymbol{g}_{(i,k)} - \boldsymbol{g}_{(k,j)} &= \left(\frac{\boldsymbol{m}_i + \boldsymbol{m}_k}{2} - \frac{\boldsymbol{m}_k + \boldsymbol{m}_{j}}{2}\right)_{\boldsymbol{m} \in \boldsymbol{M}} \nonumber \\
			&=   0.5 \left( \llbracket{}\boldsymbol{m}_i > \boldsymbol{m}_j \rrbracket{} - \llbracket{}\boldsymbol{m}_j > \boldsymbol{m}_i \rrbracket{} \right)_{\boldsymbol{m} \in \boldsymbol{M}} \nonumber \\
			& = 0.5 \left( S_{(i,j)}^\win -  S_{(i,j)}^\loss \right)
		\label{eqn1arm}
		\end{align} 
 So, $\boldsymbol{g}_{(i,k)} - \boldsymbol{g}_{(k,j)}$ falls in the row space of the signal matrix of the action $(i,j)$ and hence in the row space of the signal matrix of the neighborhood action set. (refer definition \ref{obs})

		\item Actions $(i,k)$ and $(j,k)$. Similarly, $\boldsymbol{g}_{(i,k)} - \boldsymbol{g}_{(j,k)}$ = $ 0.5    S_{(i,j)}^\win - 0.5  S_{(i,j)}^\loss $.
	\end{enumerate}
\item No common arm ($i \neq i' \neq  j \neq j'$): In this case, 
\begin{align*}
\begin{split}
\boldsymbol{g}_{(i,j)} - \boldsymbol{g}_{(i',j')} &= \boldsymbol{g}_{(i,j)} - \boldsymbol{g}_{(i,j')} + \boldsymbol{g}_{(i,j')}  - \boldsymbol{g}_{(i',j')} \\
& = 0.5 \left(  S_{(j,j')}^\win -  S_{(j,j')}^\loss + S_{(i,i')}^\win -  S_{(i,i')}^\loss \right)  \quad \quad  \quad \quad \text{Using equation \eqref{eqn1arm}}
\end{split}
\end{align*}
\end{itemize}
\noindent
Hence, for any pair of actions $(i,j)$ and $(i',j')$, $\boldsymbol{g}_{(i,j)} - \boldsymbol{g}_{(i',j')}$ falls in the row space of the signal matrix of the neighborhood action set i.e. $\boldsymbol{g}_{(i,j)} - \boldsymbol{g}_{(i',j')} \in \operatorname{Im} S^\intercal_{((i,j)(i',j'))}$ and therefore it is locally observable. So, by extension, the binary dueling bandit problem is locally observable and hence we arrive at the following corollary.
\end{proof}
\begin{corollary}
According to the hierarchy described in theorem \ref{thm1}, the binary dueling bandit problem is \emph{easy} and its regret is $ \tilde\Theta(\sqrt T)$.
\end{corollary}

\section{Partial monitoring algorithms and their use for dueling bandits}\label{PM_algorithms}
{\sc FEEDEXP3} by \cite{conf/colt/PiccolboniS01} was the first algorithm for finite partial monitoring games. For its application, there is an important pre-condition -- existence of a matrix $\mathcal{B}$ such that $\mathcal{B} \mathcal{H}  = \mathcal{G}$. We prove by contradiction that such a matrix $\mathcal{B}$ doesn't exist for the dueling bandit problem. 
Let's assume $\mathcal{B}$ exists. Therefore, for any action $(i,j) \in \boldsymbol{N}$ and any outcome vector 
$\boldsymbol{m} \in  \boldsymbol{M}$,
\[ \mathcal{G}((i,j),\boldsymbol{m}) = \sum_{i',j'=1}^{K} \mathcal{B}_{((i,j)(i',j'))} \cdot \mathcal{H}_{((i',j')(\boldsymbol{m}))}   \]
Consider $ \boldsymbol{m}=0\dots0$,
i.e. the gain of every arm is $0$. In this case, the gain of any action $(i,j)$ is $0$ and the feedback for every action is $\tie$, therefore 
\begin{equation}
 0 = \sum_{i',j'=1}^{K} \mathcal{B}_{((i,j)(i',j'))} \cdot \tie 
\label{eq2}
\end{equation}
Now consider $ \boldsymbol{m}=1\dots1$, i.e. the gain of every arm is $1$. In this case, the gain of any action $(i,j)$ is $1$ and feedback of every action is $\tie$, therefore 
\begin{equation}
1 = \sum_{i',j'=1}^{K} \mathcal{B}_{((i,j)(i',j'))} \cdot \tie
\label{eq3}
\end{equation}
Eq. \ref{eq2} and eq. \ref{eq3} reach a contradiction, therefore our assumption that $\mathcal{B}$ exists is incorrect.
\noindent
Fortunately, the authors also provide a general algorithm which performs several matrices transformations to sidestep this pre-condition. These transformations are studied thoroughly in \citep{Bartok2012phd}.

{\sc BALATON} by \cite{Bartók11minimaxregret}, {\sc CBP}-vanilla and {\sc CBP} by \cite{Bartok2012phd} belong to the family of algorithms for the locally observable PM games as does {\sc GLOBAL-EXP3} by \cite{Bartok2013}. Although, for {\sc GLOBAL-EXP3}, its regret bound of $\mathcal{\tilde{O}}(\sqrt{N'T})$ does not directly depend on the number of actions, but rather on the structure of games as $N'$ is the size of the largest \emph{point-local game}. We can however provide a counter-example for utility-based dueling bandits where $N' \approx K^2$ in the following way.

We use the notations from \cite{Bartok2013}. Consider a $p$ in the probability simplex $\Delta_{|M|}$ where all the arms have maximal gains. For this $p$, all the actions are optimal therefore this point belongs to all the cells in the cell-decomposition. Hence, according to definition 6 in \cite{Bartok2013}, there exists a point-local game consisting of all the $K(K+1)/2$ non-duplicate actions. Therefore the upper bound of GLOBALEXP3 translates to $\tilde{O}(K\sqrt{T})$ for utility-based dueling bandits.

The following table summarizes the salient features of these PM algorithms. We can clearly see that none of them, except {\sc REX3}, is optimal with respect to the number of actions $N$. Please note that for the dueling bandits problem, $N$ $\approx$  $K^2$.  
\begin{table}[h]
\begin{center}
\caption{Summary of PM algorithms}
\label{tab:taxonomy}
\begin{tabular} {c@{\quad}lc@{\quad}lc@{\quad}lc}
\bf Algorithm & \bf Setting & \bf Optimality & \bf Regret \\ \Xhline{4\arrayrulewidth}
{\sc FEEDEXP3} (\cite{conf/colt/PiccolboniS01}) & Adversarial & Not in $T$ or $N$ & $\mathcal{\tilde{O}}(T^{2/3} K)$\\ \hline
{\sc BALATON} (\cite{Bartók11minimaxregret})& Stochastic &  Not in $T$ or $N$  & $\mathcal{\tilde{O}}(K\sqrt{T})$ \\ \hline
{\sc CBP} (\cite{Bartok2012phd}) & Stochastic & in $T$, not in $N $ &  $\mathcal{\tilde{O}}(K^2 logT)$\\ \hline
{\sc GLOBAL-EXP3} (\cite{Bartok2013}) & Adversarial & in $T$, not in $N$ & $\mathcal{\tilde{O}}(K\sqrt{T})$\\ \hline
{\sc SAVAGE} (\cite{urvoy13}) & Stochastic &  in $T$, not in $N$ & $\mathcal{O}(K^2 logT)$\\ \hline
{\sc Neighborhood Watch} (\cite{Foster2011}) & Adversarial & in $T$, not in $N$ & $\mathcal{\tilde{O}}(K\sqrt{T})$\\ \hline
{\sc REX3} (\cite{GajaneICML2015}) & Adversarial & in $T$ and $N$ & $\mathcal{\tilde{O}}(\sqrt{KT})$ \\
 \Xhline{4\arrayrulewidth}
\end{tabular}
\end{center}
\end{table}

\section{Conclusion}
In this article, we studied the dueling bandit problem as an instance of the partial monitoring problem. We proved that the binary dueling bandit problem is a locally observable game and hence falls in the easy category of the partial monitoring games. We also looked at the some of the existing partial monitoring algorithms and their optimality with respect to both time and the number of actions.


\newpage

\appendix

\begin{table}[h]
\begin{center}
\caption{Notation table}
\label{tab:notation}
\begin{tabular}{c@{\quad}l@{\quad}l}
\bf Notation & \bf Description \\  \Xhline{4\arrayrulewidth}
$K$ & Number of arms & \\ \hline
$t$ & Time index & \\ \hline
$T$ & Time horizon & \\ \hline
$R_T$ & Cumulative regret after time $T$ & \\  \hline
$\mathds{E}_{\sim{}\pi}(\ldots)$ & Expectation according to $\pi$ &\\ \hline
$\boldsymbol{N}$ & set of actions & \\ \hline
$\boldsymbol{M}$ & set of outcomes & \\ \hline
$\boldsymbol{m}$ & outcome vector $\in \boldsymbol{M}$ & \\ \hline
$\mathcal{L}$ & loss function/matrix & \\ \hline
$\mathcal{G}$ & gain function/matrix & \\ \hline
$\mathcal{H}$ & feedback function/matrix & \\ \hline
$\boldsymbol{l}_i$ & loss vector: column vector consisting of $i^{th}$ row in $\mathcal{L}$ & \\ \hline
$\boldsymbol{g}_i$ & gain vector: column vector consisting of $i^{th}$ row in $\mathcal{G}$ & \\ \hline
$C_i$ & Cell of action $i$\\ \hline
$|.|$ & size of set . & \\ \hline
$\mathbb{R}$ & Set of real numbers \\ \hline
\end{tabular}
\end{center}
\end{table}

\bibliography{references}
\end{document}